\documentclass{article}
\PassOptionsToPackage{numbers, comma, sort}{natbib}
\usepackage[preprint]{neurips_2019}
\usepackage{amsmath,amsthm,amssymb,mathtools,graphicx,enumitem,hyphenat,float}
\usepackage{mathabx, hyperref}
\usepackage{import}
\usepackage{algorithm, algorithmic}
\newcommand\ec[2][]{\ensuremath{\mathbb{E}_{#1} \left[#2\right]}}
\newcommand\ecn[2][]{\ec[#1]{\norm{#2}^2}}

\newcommand\br[1]{\left ( #1 \right )}
\newcommand\pbr[1]{\left \{ #1 \right \} }

\newcommand{\R}{\mathbb{R}}

\DeclareMathOperator*{\argmin}{\arg\!\min}

\newcommand{\norm}[1]{\left\lVert#1\right\rVert}
\newcommand{\sqn}[1]{{\left\lVert#1\right\rVert}^2}

\newcommand{\eqdef}{\overset{\text{def}}{=}}

\newcommand{\cC}{{\cal C}}
\newcommand{\cO}{{\cal O}}
\newcommand{\E}[1]{{\rm \mathbb E} \left[ #1 \right]}
\makeatletter
\newsavebox\myboxA
\newsavebox\myboxB
\newlength\mylenA

\usepackage{tcolorbox}
\usepackage{pifont}
\definecolor{mydarkgreen}{RGB}{39,130,67}

\definecolor{mydarkred}{RGB}{192,47,25}

\newcommand*\overbar[2][0.75]{%
    \sbox{\myboxA}{$\m@th#2$}%
    \setbox\myboxB\null% Phantom box
    \ht\myboxB=\ht\myboxA%
    \dp\myboxB=\dp\myboxA%
    \wd\myboxB=#1\wd\myboxA% Scale phantom
    \sbox\myboxB{$\m@th\overline{\copy\myboxB}$}%  Overlined phantom
    \setlength\mylenA{\the\wd\myboxA}%   calc width diff
    \addtolength\mylenA{-\the\wd\myboxB}%
    \ifdim\wd\myboxB<\wd\myboxA%
       \rlap{\hskip 0.5\mylenA\usebox\myboxB}{\usebox\myboxA}%
    \else
        \hskip -0.5\mylenA\rlap{\usebox\myboxA}{\hskip 0.5\mylenA\usebox\myboxB}%
    \fi}
\makeatother

\theoremstyle{definition}
\newtheorem{theorem}{Theorem}
\newtheorem{corollary}{Corollary}

\newtheorem{asm}{Assumption}

\newtheorem{lemma}{Lemma}

\usepackage[colorinlistoftodos,bordercolor=orange,backgroundcolor=orange!20,linecolor=orange,textsize=scriptsize]{todonotes}

\theoremstyle{definition}
\begin{document}

\title{Gradient Descent with Compressed Iterates}
% \author{Ahmed Khaled \And Peter Richt\'arik}
 \author{Ahmed Khaled\thanks{Work done during an internship at KAUST.} \\ Cairo University \\ \texttt{akregeb@gmail.com} \And Peter Richt\'arik \\ KAUST\thanks{King Abdullah University of Science and Technology, Thuwal, Saudi Arabia} \\ \texttt{peter.richtarik@kaust.edu.sa}}
\maketitle

\begin{abstract} We propose and analyze a new type of stochastic first order method: gradient descent with compressed iterates (GDCI). GDCI in each iteration first compresses the current iterate using a lossy randomized compression technique, and subsequently takes a gradient step. This method is a distillation of a key ingredient in the current practice of federated learning, where a model needs to be compressed by a mobile device before it is sent back to a server for aggregation.  Our analysis provides a step towards closing the gap between the theory and practice of federated learning, and opens the possibility for many extensions.
\end{abstract}

\section{Introduction}

Federated learning is a machine learning setting where the goal is to learn a centralized model given access only to local optimization procedures distributed over many devices \cite{Yang19, McMahan17, LiSahu19}. This situation is common in large-scale distributed optimization involving many edge devices, and common challenges include data heterogeneity \cite{Zhao18}, privacy \cite{Geyer17}, resource management \cite{Nishio19, WangTuor18}, and system heterogeneity as well as communication efficiency \cite{LiSahu19, Konecny16}. The most commonly used optimization methods in federated learning are variants of distributed gradient decent, stochastic gradient and gradient-based methods such as Federated Averaging \cite{LiSahu19}. 

The training of high-dimensional federated learning models \cite{Konecny16, FEDOPT} reduces to solving an optimization problem of the form
\begin{equation*}
    \label{eq:optimization-problem}
    x_\ast = \argmin_{x \in \R^d} \left[ f(x)\eqdef \frac{1}{n} \sum_{i=1}^n f_i(x) \right],
\end{equation*}
where $n$ is the number of consumer devices (e.g., mobile devices), $d$ is the number of parameters/features of the model, and $f_i: \R^d \to \R$ is a loss function that depends on the private data stored on the $i$th device.  The simplest benchmark method\footnote{Which is a starting point for the development of more advanced methods.} for solving this problem is gradient descent, which performs updates of the form
\[x_{k+1} = \frac{1}{n} \sum_{i=1}^n \left(x_k - \gamma \nabla f_i(x_k) \right).\]
That is, all nodes in parallel first perform a single gradient descent step starting from $x_k$ based on their local data, the resulting models are then communicated to a central machine/aggregator, which performs model averaging. The average model is subsequently communicated back to all devices, and the process is repeated until a model of a suitable quality is found. 

Practical considerations of federated learning impose several constraints on the feasibility of this process. First, due to geographical and other reasons, model averaging is performed in practice on a subset of nodes at a time only. Second, in a hope to address the communication bottleneck, each device is typically allowed to take multiple steps of gradient descent or stochastic gradient descent before aggregation takes place. Methods of this type are known as {\em local} methods in the literature \cite{Konecny16, McMahan17}. Third, in the large dimensional case, the models are typically compressed \cite{Konecny16, Caldas18} by the devices before they are communicated to the aggregator, and/or by the aggregator before the averaged model is pushed to the devices. 

In distributed stochastic gradient methods, the cost of gradient communication between training nodes and the master node or parameter server has been observed to be a significant performance bottleneck. As a result, there are many algorithms designed with the goal of reducing communication in stochastic gradient methods: including SignSGD (1-bit quantization) \cite{Bernstein18}, TernGrad (ternary quantization) \cite{WeiWen17}, QSGD \cite{Alistarh16}, DIANA (with arbitrary quantization) \cite{DIANA2}, ChocoSGD \cite{Koloskova19}, and others, see e.g.\ \cite{LinHan17, BenNun18} and the references therein.  Among compression operators used in quantized distributed stochastic gradient methods, compression operators satisfying Assumption~\ref{asm:compression-operator} are ubiquitous and include natural compression \cite{Horvath19NaturalComp}, dithering \cite{Goodall51, Roberts62}, natural dithering \cite{Horvath19NaturalComp}, sparsification \cite{Stich18}, ternary quantization \cite{WeiWen17}, and others.   As an alternative to costly parameter server communication, decentralized methods can achieve better communication efficiency by using inter-node communication. \cite{Lian17, BenNun18} and combinations of decentralization and gradient quantization have been studied in recent work, see e.g. \cite{Tang18, Koloskova19, Koloskova19NonConvex}.  Another line of work focused on local stochastic gradient methods that communicate only intermittently and average models, such as Local SGD \cite{Stich2018, Wang18, Wang2019, TaoLin19} and Federated Averaging \cite{McMahan17}, and combinations of such methods and update quantization (where the sum of gradients over an epoch is quantized) have also been explored in the literature \cite{Jiang18, Basu2019}.

%Due to the relatively high cost of communication compared to computation in the federated learning setting, a plethora of techniques have been devised to achieve communication efficiency: including the use of intermittent communication and model averaging \cite{McMahan17}, as well as various forms of compression or quantization \cite{Caldas18}. 
%
%These techniques are often heurestically combined together in practice, making their theoretical analysis difficult. Iterate quantization is one such technique to save communication and storage costs, but little is understood about its convergence guarantees from a theoretical perspective. 

{\bf Gaps in theory of federated learning.}  There are considerable gaps in our theoretical understanding of federated learning algorithms which use these tricks. For instance, until very recently \cite{localGD}, no convergence results were known for the simplest of all local methods---local gradient descent---in the case when the functions $f_i$ are allowed to be arbitrarily different, which is a requirement of any efficient federated learning method since data stored on devices of different users can be arbitrarily heterogeneous. Further, while there is ample work on non-local methods which communicate compressed gradients  \cite{Alistarh16, BenNun18, Wangni18, WeiWen17, Horvath19NaturalComp}, including  methods which perform variance-reduction to remove the variance introduced by compression \cite{DIANA, DIANA2, 99percent}, to the best of our knowledge there is little work on methods performing iterative model compression, and the only one we are aware of is the very recent work in \cite{Reisizadeh19} which is a distributed variant of SGD that quantizes iterate communication. To remove the iterate quantization variance, they do a relaxation over time in the iterates and no results are provided when averaging across time is not performed. Similar statements can be made about our understanding of other elements of current practice.

{\bf Iterative model compression.}  In this paper we focus on a single element behind efficient federated learning methods---iterative model compression---and analyze it in isolation.  Surprisingly, we are not aware of any theoretical results in this area, even in the simplest of settings: the case of a single device ($n=1$) with a smooth and strongly convex function.  

Motivated by the desire to take step towards bridging the gap between theory and practice of federated learning, in this paper we study the algorithm 
\begin{equation}
    \label{eq:pgdci-update}
    x_{k+1} = \cC(x_k) - \gamma \nabla f(\cC(x_k)),
\end{equation}
where $\cC: \R^d \to \R^d$ is a sufficiently well behaved unbiased stochastic compression operator (see Assumption~\ref{asm:compression-operator} for the definitions). We call this method {\em gradient descent with compressed iterates (GDCI).} The update in equation \eqref{eq:pgdci-update} captures the use of compressed iterates/models in place of full iterates on a single node. Clearly, this method should be understood if we are to tackle the more complex realm of distributed optimization for federated learning, including the $n>1$ setting, partial participation and local variants.  We believe that our work will be starting point of healthy research into iterative methods with compressed iterates. One of the difficulties in analyzing this method is the observation that $\nabla f(\cC(x))$ is not an unbiased estimator of the gradient, even if $\cC$ is unbiased.

% \import{algs/}{pgdci}

%\peter{Will continue from here soon}
%
%\begin{equation*}
%    x_{k+1} = x_k - \gamma g_k,
%\end{equation*}
%where $\gamma > 0$ is a stepsize, $g_k = \nabla f(x_k)$ or some stochastic estimate of this gradient. 
%
%
%
%
%In distributed formulations where there are many devices participating in the optimization process, the new iterate $x_{k+1}$ is either formed by aggregating local iterates computed by several nodes or by aggregating local gradients. While the latter method has been analyzed in many settings, see e.g.\ \cite{Alistarh16, BenNun18, Wangni18, WeiWen17, Horvath19}, there is little work on the former despite being used in recent work on Federated Learning such as \cite{Caldas18}.  

\section{Assumptions and Contributions}

In this work we assume that $f$ is smooth and strongly convex:
\begin{asm}
    \label{asm:smoothness-and-convexity}
The function $f: \R^d \to \R$ is $L$-smooth and $\mu$-strongly convex: that is, there exists $L \geq \mu > 0$ such that
$$\mu \norm{x-y} \leq \norm{\nabla f(x) - \nabla f(y)} \leq L \norm{x - y}$$ for all  $x,y\in \R^d.$    We define the condition number of $f$ as $\kappa \eqdef \frac{L}{\mu}$.
\end{asm}

We make the following assumptions on the compression operator:
\begin{asm}
    \label{asm:compression-operator}
    The compression operator $\cC:\R^d \to \R$ is unbiased, i.e., 
    \begin{align}
        \label{eq:asm:compression-operator-unbiased}
        \ec{\cC(x) \mid x} &= x, \qquad \forall x\in \R^d,
    \end{align}
   and there exists $\omega\geq 0$ such that its variance is bounded as follows
    \begin{align}
        \label{eq:asm:compression-operator-variance}
        \ec{\sqn{\cC(x) - x} } &\leq \omega \sqn{x}, \qquad \forall x\in \R^d.
    \end{align}
\end{asm}

Our main contribution is to show that the iterates generated by GDCI  (Algorithm~\eqref{eq:pgdci-update}) converge linearly, at the same rate as gradient descent, to a neighbourhood of the solution $x_\ast$ of size $\cO(\kappa \omega)$, where $\kappa = L/\mu$ is the condition number of the solution.

\begin{theorem}
    \label{thm:main-convergence-theorem}
    Suppose the Assumptions~\ref{asm:smoothness-and-convexity} and \ref{asm:compression-operator} hold. Suppose that GDCI is run with a constant stepsize $\gamma > 0$ such that $\gamma \leq \frac{1}{2 L}$ and assume that the compression coefficient $\omega \geq 0$ satisfies
    \begin{align}
        \label{eq:omega-bound}
        \frac{4 \omega}{\mu} \leq \frac{1 - 2 \gamma L}{2 \gamma L^2 + \frac{2}{\gamma} + L - \mu}.
    \end{align}
    Then,
    \begin{align}
        \label{eq:convergence-rate}
        \ecn{x_k - x_\ast} \leq \br{1 - \gamma \mu}^{k} \sqn{x_0 - x_\ast} + \frac{2 \omega}{\mu} \br{4 \gamma L^2 + \frac{4}{\gamma} + L - \mu} \sqn{x_\ast}.
    \end{align}
\end{theorem}
The proof of Theorem~\ref{thm:main-convergence-theorem} is provided in the supplementary material. The following corollary gives added insight:

\begin{corollary}
    In Theorem~\ref{thm:main-convergence-theorem}, suppose that $\gamma = \frac{1}{4 L}$ and that $\omega \leq \frac{1}{73 \kappa}$, then the bound in Equation~\eqref{eq:omega-bound} is satisfied and substituting in \eqref{eq:convergence-rate} we have,
\[
        \ecn{x_k - x_\ast} \leq \br{1 - \frac{1}{4 \kappa}}^k \sqn{x_0 - x_\ast} + 2 \omega \br{18 \kappa - 1} \sqn{x_\ast}.
\]
    This is the same rate as gradient descent, but only to a $\cO(\kappa \omega)$ neighbourhood (in squared distances) of the solution.
\end{corollary}

Note that if we want to set the  neighbourhood to  $\cO(1)$, then we should have $\omega = \cO\br{\kappa^{-1}}$. While this seems to be a pessimistic bound on the compression level possible, we note that in practice compression is done only intermittently (this could be modelled by an appropriate choice of $\cC$; more on this below) or in a combination with averaging (which naturally reduces the variance associated with quantization).  In practical situations where averaging is not performed, such as the quantization of server-to-client communication, high compression levels do not seem possible without serious deterioration of the accuracy of the solution \cite{Caldas18}, and our experiments also suggest that this is the case.

% \pagebreak

\begin{figure}[t]
    \centering
    \includegraphics[width=13cm]{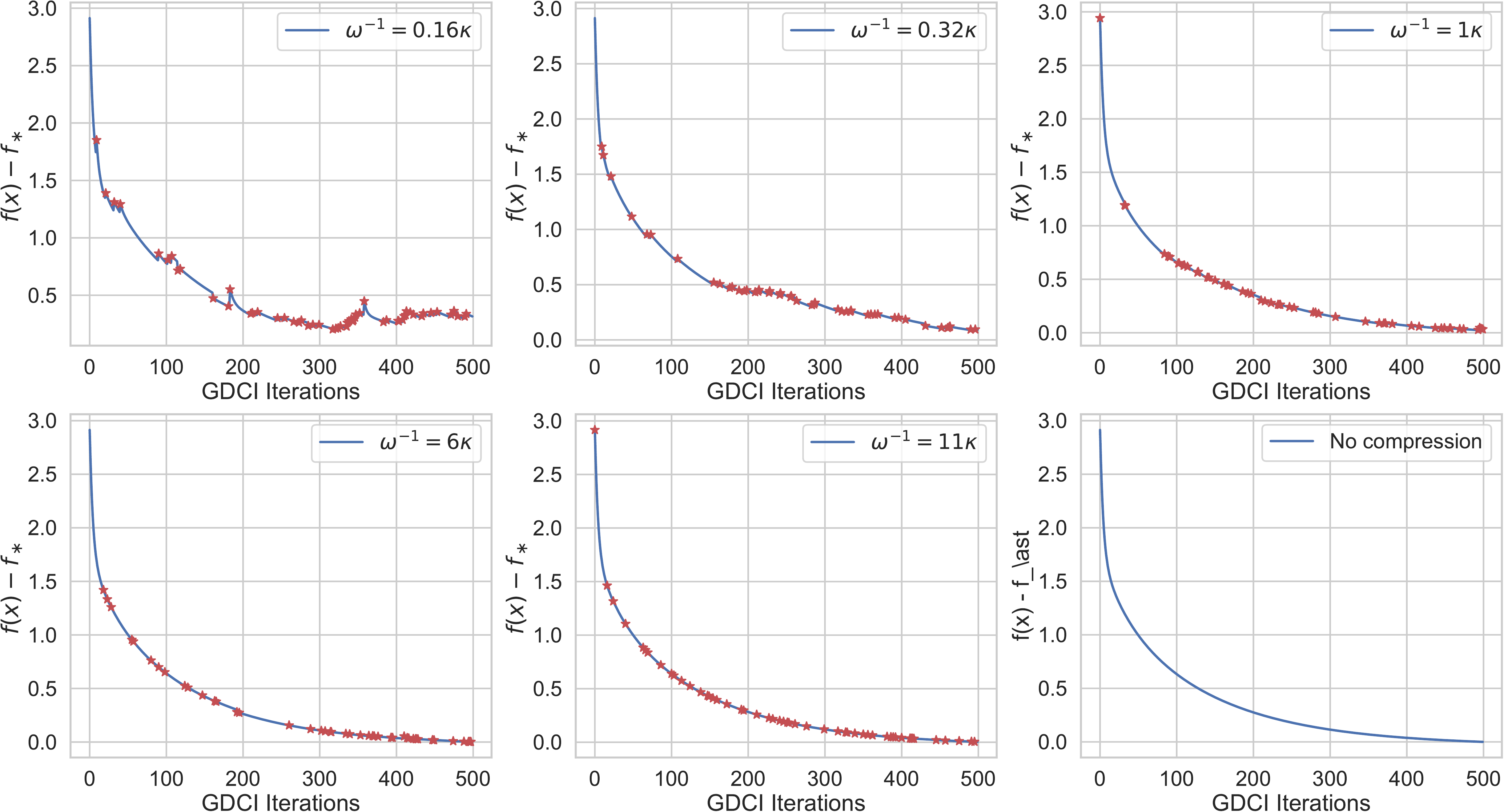}
    \caption{GDCI  as $\omega$ varies for the ``a7a'' dataset. Red star indicates $\cC$ was applied in that iteration.}
    \label{fig:a7a-plot-1}
\end{figure}

\section{Experiments}
To confirm our theoretical results, we experiment with a logistic regression problem:
\begin{equation}
    \min_{w} \pbr{ f(w) = \frac{1}{n} \sum_{i=1}^{n} \log(1 + \exp\br{- b_i x_i^\top w})+ \frac{\mu}{2} \sqn{w} },
\end{equation}
where $x_i \in \R^d$ and $y_i \in \R$ are the data samples for $i \in [n]$. We consider the ``a7a'' and ``a5a'' datasets from the UCI Machine Learning repository \cite{Dua19} with $n = 16100$ for ``a7a'' and $n = 6414$ for ``a5a'' and $d = 123$ in both cases. We set the regularization parameter $\mu = 0.02$ and estimate $\kappa \simeq 161$ for the ``a7a'' dataset and $\kappa \simeq 65$ for the ``a5a'' dataset. We consider the random sparsification operator, where each coordinate is independently set to zero according to some given probability. That is, given $p \in (0, 1]$ we have for $c: \R \to \R$,
\begin{equation}
    c(x) = \begin{cases}
            \frac{x}{p} & \text{ with probability } p \\
            0 & \text{ with probability } 1 - p 
        \end{cases}
\end{equation}
and we define $\cC: \R^d \to \R$ by $(\cC(x))_{i} = c(x_i)$ for all $i \in [n]$ independently. Note that for this quantization operator $\cC$ we have that Assumption~\ref{asm:compression-operator} is satisfied with $\omega = \frac{1-p}{p}$. 

\begin{figure}[t]
    \centering
    \includegraphics[width=13cm]{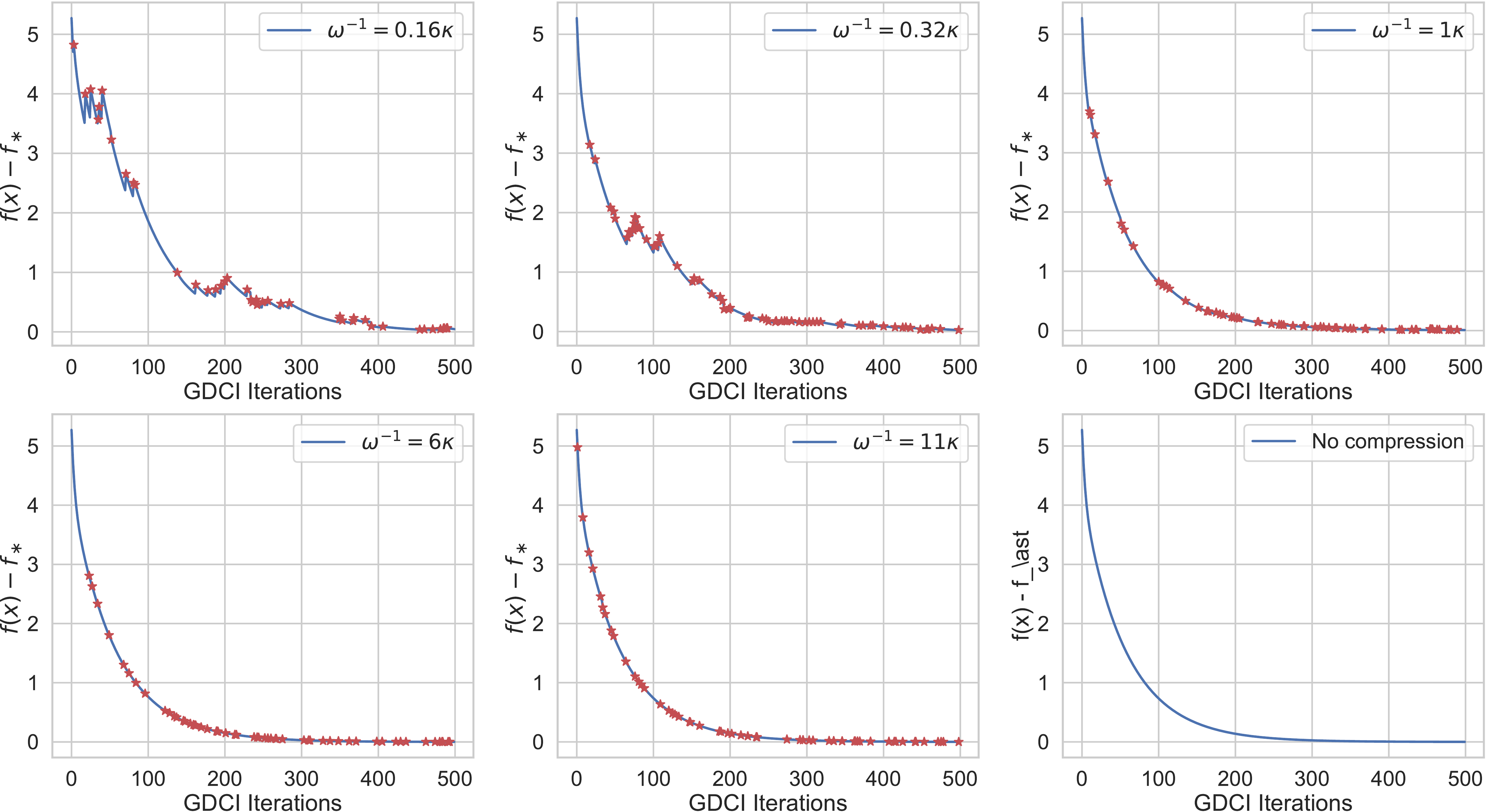}
    \caption{GDCI as $\omega$ varies for the ``a5a'' dataset. Red star indicates $\cC$ was applied in that iteration.}
    \label{fig:a5a-plot-1}
\end{figure}

To model intermittent quantization experimentally, we apply the quantization operator $C$ with probability $1/10$ and keep the iterate as it is with probability $9/10$. We vary $\omega$ as $\frac{1}{\alpha \kappa}$ for various settings of $\alpha$. The results are shown for the ``a7a'' dataset are shown in Figure~\ref{fig:a7a-plot-1} and for the ``a5a'' dataset in Figure~\ref{fig:a5a-plot-1}. 

The results of Figure~\ref{fig:a7a-plot-1} show that for $\omega$ small enough the effect on convergence is negligible, but the effect on the error at convergence becomes noticeable at $\omega \in \pbr{ \frac{4}{25 \kappa}, \frac{8}{25 \kappa}  }$ and we have observed divergent behavior for larger values of $\omega$. Similar behavior is observed for the plots in Figure~\ref{fig:a5a-plot-1}.

% \ahmed{This work has actually been put up on arXiv at the end of July / start of August. Wasn't there when we started working on this.}

% \pagebreak
\bibliography{gdci}

\clearpage
\part*{Gradient Descent with Compressed Iterates \\ Supplementary Material}

\section{Basic Inequalities}
We will often use the bound 
\begin{equation}\label{eq:simple}
    \norm{a+b}^2 \leq 2\norm{a}^2 + 2\norm{b}^2.
\end{equation}

If $f$ is an $L$--smooth and convex function, then the following inequalities hold
\begin{equation}\label{eq:nb9ugf9f} \norm{\nabla f(x) - \nabla f(y)} \leq L \norm{x-y}, \end{equation}
\begin{equation} 
    \label{eq:nihf98hfssdsf3}
    f(x) \leq f(y) + \langle \nabla f(y), x-y \rangle + \frac{L}{2} \norm{x-y}^2, 
\end{equation} 
\begin{equation} 
    \label{eq:n98dg90hgsbvs}  
    f(y) + \langle \nabla f(y), x-y \rangle + \frac{1}{2L} \norm{\nabla f(x) - \nabla f(y)}^2 \leq f(x). 
\end{equation} 
If $f$ is $\mu$-strongly convex, then the following inequality holds
\begin{equation}
    \label{eq:strong_convex_ineq}
    f(y) \geq f(x) + \langle \nabla f(x), y-x \rangle + \frac{\mu}{2}\norm{y-x}^2 , \quad \forall x,y\in \R^d.
\end{equation}
We define $\delta(x) \eqdef \cC(x)-x$.

\section{Five Lemmas}

In the first lemma we give an upper bound on the variance of the compression operator $\cal C$.

\begin{lemma}
    \label{lemma:compression-functional-values}
    Suppose that a compression operator $\cC: \R^d \to \R^d$ satisfies Assumption~\ref{asm:compression-operator}, then
    \begin{equation}
        \label{eq:lma-compression-functional-values}
        \ecn{\cC(x) - x} \leq 2 \alpha \br{f(x) - f(x_\ast)} + \beta,
    \end{equation}
    with $\alpha = \frac{2 \omega}{\mu}$ and $\beta = 2 \omega \sqn{x_\ast}$.
\end{lemma}
\begin{proof}
    First, note that $\sqn{x} \leq 2 \sqn{x - x_\ast} + 2 \sqn{x_\ast}$. If $f$ is $\mu$-strongly convex, then by \eqref{eq:strong_convex_ineq} we have that $\sqn{x - x_\ast} \leq \frac{2}{\mu} \br{f(x) - f(x_\ast)}$, and putting these inequalities together, we arrive at
$$
        \ecn{C(x) - x} \leq \omega \sqn{x} \leq 2 \omega \sqn{x - x_\ast} + 2 \omega \sqn{x_\ast} \leq \frac{4 \omega}{\mu} \br{f(x) - f(x_\ast} + 2 \omega \sqn{x_\ast}.
$$
\end{proof}

Our second lemma is an extension of several standard inequalities which trivially hold (for $L$-smooth and convex functions) in the case of no compression, i.e.,  $\delta(x)\equiv 0$, to a situation where a compression is applied. Indeed, notice that \eqref{eq:no9f8ghshis} is a generalization of \eqref{eq:nb9ugf9f}, and the second inequality in \eqref{eq:bf8g7d8vd}  is a generalization of \eqref{eq:nihf98hfssdsf3}.

\begin{lemma} 
    If the compression operator $\cC$ satisfies \eqref{eq:asm:compression-operator-unbiased} and $f$ is convex and $L$-smooth, then
    \begin{equation}
        \label{eq:no9f8ghshis} 
        \E{\norm{\nabla f(x+ \delta(x)) - \nabla f(y)}^2}  \leq L^2 \left( \norm{x-y}^2 + \E{\norm{\delta(x)}^2}\right), \quad \forall x,y\in \R^d.
    \end{equation}
    And for all $x, y \in \R^d$ we also have,
    \begin{equation}
        \label{eq:bf8g7d8vd} 
        f(x) \leq \E{ f(x+\delta(x))} \leq f(y) + \langle \nabla f(y), x-y \rangle + \frac{L}{2}\norm{x-y}^2 + \frac{L}{2}\E{\norm{\delta(x)}^2}.
    \end{equation}
\end{lemma}
\begin{proof}
    Fix $x$ and let $\delta = \delta(x)$. Inequality \eqref{eq:no9f8ghshis} follows  from Lipschitz continuity of the gradient, applying expectation and using \eqref{eq:asm:compression-operator-unbiased}:
    \begin{eqnarray*}
        \E{\norm{\nabla f(x+ \delta) - \nabla f(y)}^2}  \leq L^2 \E{ \norm{x+ \delta - y}^2 } \overset{\eqref{eq:asm:compression-operator-unbiased}}{=} L^2\left( \norm{x-y}^2 + \E{\norm{\delta}^2}\right).
    \end{eqnarray*}
    The  first inequality in \eqref{eq:bf8g7d8vd} follows by applying Jensen's inequality and using \eqref{eq:asm:compression-operator-unbiased}.  Since $f$ is $L$--smooth, we have
    \begin{eqnarray*}
        \E{ f(x+\delta)} &\leq & \E{ f(y) + \langle \nabla f(y), x+ \delta - y\rangle + \frac{L}{2}\norm{x+ \delta-y}^2} \\
        &\overset{\eqref{eq:asm:compression-operator-unbiased}}{=}& f(y) + \langle \nabla f(y), x - y\rangle +  \frac{L}{2}\norm{x-y}^2 + \frac{L}{2}\E{\norm{\delta}^2}.
    \end{eqnarray*}
\end{proof}

\begin{lemma}  
    If the compression operator $\cC$ satisfies \eqref{eq:asm:compression-operator-unbiased}, then for all $x, y \in \R^d$
    \begin{equation}
        \label{eq:buvdyvdf8d87}  
        \E{\norm { \frac{\delta(x)} {\gamma}-\nabla f(x+ \delta(x))}^2} \leq  2 \norm{\nabla f(y)}^2 + 2 L^2 \norm{x-y}^2 + 2\left(L^2 + \frac{1}{\gamma^2} \right)\E{\norm{\delta(x)}^2}.
    \end{equation}
\end{lemma}
\begin{proof}
    Fix $x$, and let $\delta = \delta(x)$. Then for every $y\in \R^d$ we can write
    \begin{eqnarray}
        \E{\norm { \frac{\delta} {\gamma}-\nabla f(x+ \delta)}^2}  &=& \E{\norm { \frac{\delta} {\gamma} -\nabla f(y) + \nabla f(y)-\nabla f(x+ \delta)}^2}  \notag \\
        &\overset{\eqref{eq:simple}}{\leq} & 2 \E{\norm{\frac{\delta} {\gamma} -\nabla f(y)}^2} +  2 \E{\norm{ \nabla f(y)-\nabla f(x+ \delta)}^2}\notag  \\
        &\overset{\eqref{eq:no9f8ghshis}}{\leq} & 2 \E{ \frac{1}{\gamma^2}\norm{\delta}^2 - \frac{1}{\gamma}\langle \delta, \nabla f(y)\rangle + \norm{\nabla f(y)}^2} \\
        && \qquad + 2 L^2 \left( \norm{x-y}^2 + \E{\norm{\delta}^2}\right) \notag \\
        %%%
        & \overset{\eqref{eq:asm:compression-operator-unbiased}}{\leq} & \frac{2}{\gamma^2} \E{\norm{\delta}^2} + 2 \norm{\nabla f(y)}^2\\ 
        && \qquad  + 2 L^2 \left( \norm{x-y}^2 + \E{\norm{\delta}^2}\right) .   \notag
    \end{eqnarray}
\end{proof}

The next lemma generalizes the strong convexity inequality \eqref{eq:strong_convex_ineq}. Indeed,  \eqref{eq:strong_convex_ineq} is recovered in the special case $\delta(x)\equiv 0$.

\begin{lemma}
    Suppose Assumptions~\ref{asm:compression-operator} and \ref{asm:smoothness-and-convexity} hold.  Then for all $x, y \in \R^d$,
    \begin{equation}
        \label{eq:bui8fgfihf} 
        f(y) \geq  f(x) + \langle \E{ \nabla f(x+ \delta)}, y-x\rangle + \frac{\mu}{2} \norm{y-x}^2 - \frac{L-\mu}{2}\E{\norm{\delta(x)}^2}.
    \end{equation} 
\end{lemma}
\begin{proof} 
    Fix $x$ and let $\delta = \delta(x)$. Using \eqref{eq:strong_convex_ineq} with  $x \leftarrow x+\delta$, we get
    \[ f(y) \geq f(x+ \delta) + \langle \nabla f(x+ \delta), y - x - \delta \rangle + \frac{\mu}{2}\norm{y-x-\delta}^2. \]
    Applying expectation, we get
    \begin{eqnarray}
        f(y)  & \geq & \E{f(x+\delta)}  + \E{ \langle \nabla f(x+ \delta), y -x \rangle } - \E{\langle \nabla f(x+\delta), \delta\rangle} \notag\\
        && \qquad   + \frac{\mu}{2}\norm{y-x}^2 + \frac{\mu}{2}\E{\norm{\delta}^2}.        \label{eq:bu9dg8fsd} 
    \end{eqnarray}
    The  term $- \E{\langle \nabla f(x+\delta), \delta\rangle}$  can be estimated using $L$--smoothness and applying expectation as follows:
    \begin{eqnarray*} 
        \E{-\langle \nabla f(x+\delta), \delta\rangle}  &\geq &  \E{ f(x) - f(x+\delta) - \frac{L}{2}\norm{\delta}^2 } \\
        &=& f(x) - \E{f(x+\delta)} - \frac{L}{2}\E{\norm{\delta}^2}.
    \end{eqnarray*}
    It remains to  plug this  inequality to \eqref{eq:bu9dg8fsd}.
\end{proof}

\begin{lemma}
    \label{lem:ES-generalized} Suppose that Assumptions~\ref{asm:smoothness-and-convexity} and \ref{asm:compression-operator} hold. Then 
    \begin{equation} 
        \label{eq:bu98fg7fss} 
        \E{\norm { \frac{\delta(x)} {\gamma}-\nabla f(x+ \delta(x))}^2} \leq  4A (f(x)-f(x_*)) + 2 B, \quad \forall x\in \R^d,
    \end{equation}
    where $A = L+ \left(L^2 + \frac{1}{\gamma^2}\right) \alpha  $ and $B= 2 \left(L^2+ \frac{1}{\gamma^2}\right) \beta$ and $\alpha, \beta$ are defined in Lemma~\ref{lemma:compression-functional-values}.
\end{lemma}
\begin{proof} 
    Using \eqref{eq:buvdyvdf8d87} with $y=x$, we get 
    \begin{eqnarray*}
        \E{\norm { \frac{\delta(x)} {\gamma}-\nabla f(x+ \delta(x))}^2} & \overset{\eqref{eq:buvdyvdf8d87}}{\leq} &  2 \norm{\nabla f(x)}^2 +  2\left(L^2 + \frac{1}{\gamma^2} \right)\E{\norm{\delta(x)}^2}\\
        & \overset{\eqref{eq:n98dg90hgsbvs} +\eqref{eq:lma-compression-functional-values}}{\leq} & 4L (f(x) - f(x_*)) + 2 \left(L^2 + \frac{1}{\gamma^2} \right) \left(2 \alpha (f(x)-f(x_*)) + \beta\right) \\
        &=& 4\left(L + \left(L^2 + \frac{1}{\gamma^2} \right) \alpha \right)(f(x_k)-f(x_*)) + 2 \left(L^2 + \frac{1}{\gamma^2} \right) \beta.
    \end{eqnarray*}
\end{proof}

\section{Proof of Theorem~\ref{thm:main-convergence-theorem}}

\begin{proof}
    Let $r_k = \norm{x_k-x_*}^2$, $\delta_k = \delta(x_k)$ (hence $\cC(x_k) = x_k +\delta_k$). Then
    \begin{eqnarray*}
        r_{k+1} &=& \norm{\cC(x_k)  - \gamma \nabla f(\cC(x_k)) - x_*}^2 \\
        &= & \norm{x_k - x_* + \delta_k  - \gamma \nabla f(x_k + \delta_k)}^2  \\
        &=& r_k + 2 \langle   \delta_k - \gamma \nabla f(x_k + \delta_k), x_k -x_*\rangle + \norm{\delta_k - \gamma \nabla f(x_k + \delta_k)}^2.
    \end{eqnarray*}
    
    Taking conditional expectation, we get
    \begin{eqnarray*}
        \E {r_{k+1} \;|\; x_k} &=& r_k + 2\gamma \langle \E{\nabla f(x_k + \delta_k) \;|\; x_k}, x_* -x_k \rangle + \E{\norm{\delta_k - \gamma \nabla f(x_k + \delta_k)}^2 \;|\; x_k}\\
        &\overset{\eqref{eq:bui8fgfihf}}{\leq}& r_k + 2\gamma \left[ f(x_*) - f(x_k) -\frac{\mu}{2}\norm{x_k-x_*}^2 + \frac{L-\mu}{2}\E{ \norm{\delta_k}^2 \;|\; x_k} \right]  \\
        && \qquad + \gamma^2   \E{\norm{\frac{\delta_k}{\gamma} - \nabla f(x_k + \delta_k)}^2 \;|\; x_k}\\ 
        & =& (1-\gamma \mu) r_k - 2\gamma(f(x_k) - f(x_*)) + \gamma (L-\mu)\E{ \norm{\delta_k}^2 \;|\; x_k}   \\
        && \qquad + \gamma^2   \E{\norm{\frac{\delta_k}{\gamma} - \nabla f(x_k + \delta_k)}^2 \;|\; x_k}\\ 
        &\overset{\eqref{eq:bu98fg7fss}}{\leq}  & (1-\gamma \mu) r_k - 2\gamma(f(x_k) - f(x_*)) + \gamma (L-\mu)\E{ \norm{\delta_k}^2 \;|\; x_k}   \\
        && \qquad + 4 \gamma^2 A(f(x_k)-f(x_*)) + 2\gamma^2 B \\
        &=& (1-\gamma \mu) r_k + \gamma (4 \gamma A - 2 )(f(x_k)-f(x_*)) + 2\gamma^2 B + \gamma (L-\mu) \E{ \norm{\delta_k}^2 \;|\; x_k} \\
        & \overset{\eqref{eq:lma-compression-functional-values}}{\leq} &  (1-\gamma \mu) r_k +  \gamma (4 \gamma A - 2 )(f(x_k)-f(x_*)) + 2\gamma^2 B \\
        && \qquad + \gamma (L-\mu)\left(2\alpha (f(x_k)-f(x_*)) + \beta\right) \\
        &=&  (1-\gamma \mu) r_k + 2 \gamma(2\gamma A + \alpha (L-\mu) - 1) (f(x_k)-f(x_*))  + 2\gamma^2 B + \gamma (L-\mu) \beta,
    \end{eqnarray*}
    where $\alpha$ and $\beta$ are as in Lemma~\ref{lemma:compression-functional-values} and $A$ and $B$ are defined Lemma~\ref{lem:ES-generalized}. By assumption on $\alpha$ and $\gamma$, we have $2\gamma A + \alpha (L-\mu) \leq 1$, and hence
    $ \E {r_{k+1} \;|\; x_k}  \leq  (1-\gamma \mu) r_k +  D,$
    where $D=2\gamma^2 B + \gamma (L-\mu) \beta$. Taking expectation, unrolling the recurrence,  and applying the tower property, we get 
    \[ \E{r_k} \leq (1-\gamma \mu)^k r_0 + \frac{D}{\gamma \mu}.\]
    Writing out $D$ yields the expression for the convergence rate in \eqref{eq:convergence-rate}. For the bound on $\omega$, we first write out the definition of $A$ in $2 \gamma A + \alpha (L - \mu) \leq 1$ we have,
    \begin{align}
        \label{eq:main-thm-proof-1}
        2 \gamma \br{L + \br{L^2 + \frac{1}{\gamma^2}} \alpha } + \alpha \br{L - \mu} \leq 1 .
    \end{align}
    Rearranging terms in \eqref{eq:main-thm-proof-1} we get that,
    \begin{align*}
        \alpha \leq \frac{1 - 2 \gamma L}{2 \gamma L^2 + \frac{2}{\gamma} + L - \mu}.
    \end{align*}
    Using the fact that $\omega = \frac{\alpha \mu}{2}$ yields \eqref{eq:omega-bound}.
\end{proof}

\end{document}